\documentclass{article}
 
\usepackage{arxiv}

\usepackage[utf8]{inputenc} 
\usepackage[T1]{fontenc}    
\usepackage{url}            
\usepackage{booktabs}       
\usepackage{amsfonts}       
\usepackage{nicefrac}       
\usepackage{microtype}      
\usepackage{graphicx}
\usepackage{natbib}         
\usepackage{hyperref}       
\usepackage{doi}

\usepackage{amsmath,amssymb,amsthm}  
\usepackage{mathtools}               
\usepackage{bm}                      
\usepackage{multirow}                
\usepackage{subcaption}              
\usepackage{algorithm}               
\usepackage{algorithmic}
\usepackage{pifont, xcolor}
\usepackage{makecell}
\usepackage{tikz}

\usetikzlibrary{positioning,calc,arrows.meta}

\renewcommand{\arraystretch}{1.15}
\newcommand{\fitwidth}[1]{%
  \resizebox{\ifdim\width>\linewidth\linewidth\else\width\fi}{!}{#1}%
}

\hypersetup{
  colorlinks   = true,
  linkcolor    = [rgb]{0.10,0.20,0.55},
  citecolor    = [rgb]{0.10,0.35,0.20},
  urlcolor     = [rgb]{0.55,0.15,0.15},
  breaklinks   = true,
  bookmarksnumbered = true,
}

\providecommand{\cmark}{\textcolor{teal}{\ding{51}}}
\providecommand{\xmark}{\textcolor{red!70!black}{\ding{55}}}
\providecommand{\gap}[3]{\makecell[c]{$#1$\\[-1.5pt]{\scriptsize$[\,#2,\ #3\,]$}}}

\providecommand{\cmark}{\textcolor{teal}{\ding{51}}}
\providecommand{\xmark}{\textcolor{red!70!black}{\ding{55}}}

\theoremstyle{plain}

\newtheorem{proposition}{Proposition}

\theoremstyle{definition}

\theoremstyle{remark}

\newcommand{\comp}{\mathsf{C}}                       
\newcommand{\units}{N}                               
\newcommand{\dep}{D}                                 
\newcommand{\hier}{H}                                
\newcommand{\game}{G}                                
\newcommand{\feas}{\mathcal{F}}                      
\newcommand{\tasks}{\mathcal{T}}                     
\newcommand{\score}{\operatorname{score}}
\newcommand{\agentop}{\operatorname{Agent}}
\newcommand{\rhodel}{\rho_{\mathrm{del}}}
\newcommand{\rhopad}{\rho_{\mathrm{pad}}}
\newcommand{\val}{V}                                 
\newcommand{\prefixset}[2]{S_{#1}^{#2}}              
\newcommand{\E}{\mathbb{E}}

\title{What Is a Skill Worth? Structure-Aware Shapley Valuation of Agent Skills}

\author{%
  Tao Li$^{1}$\hspace{0.55em}
  Junfeng Liu$^{2}$\thanks{Corresponding author: \texttt{liujf@pcl.ac.cn}}\hspace{0.55em}
  Qinghua Zhao$^{3}$\hspace{0.55em}
  Yifan Li$^{2}$\hspace{0.55em}
  Lei Wang$^{2}$\hspace{0.55em}
  Bo Shao$^{4}$\hspace{0.55em}
  Xuejun Liu$^{1}$\hspace{0.55em}
  Linjun Shou$^{4}$ \\[7pt]
  {\normalfont\normalsize
   $^{1}$Nanjing University of Aeronautics and Astronautics \quad
   $^{2}$Pengcheng Laboratory} \\[2pt]
  {\normalfont\normalsize
   $^{3}$Hefei University \quad
   $^{4}$Microsoft}
}

\date{}

\renewcommand{\shorttitle}{What Is a Skill Worth? Structure-Aware Shapley Valuation of Agent Skills}

\hypersetup{
pdftitle={What Is a Skill Worth? Structure-Aware Shapley Valuation of Agent Skills},
pdfsubject={cs.AI, cs.CL, cs.LG},
pdfauthor={Tao Li, Junfeng Liu, Qinghua Zhao, Yifan Li, Lei Wang, Bo Shao, Xuejun Liu, Linjun Shou},
pdfkeywords={agent skills, Shapley value, data valuation, credit assignment, LLM agents, prompt compression},
}

\begin{document}
\maketitle

\begin{abstract}
	Agent skills are increasingly optimized by automated feedback loops, producing
long structured artifacts whose internal value remains unclear. 
 We study \emph{skill valuation}: assigning credit to the internal units of a
  fixed skill, such as rules, examples, scripts, and heuristics, under a fixed
  agent and held-out task distribution. Skill valuation differs from data or
  prompt-span valuation because skill units are structured: they may depend on
  other units, belong to a document hierarchy, trigger agent behavior, and consume
  limited prompt context.
    We introduce \textsc{SkillSV}, a structure-aware Shapley-style framework for
  skill valuation. \textsc{SkillSV} compiles a skill into units, dependencies, and
  hierarchy, so that only valid counterfactual skills are evaluated. It uses
  paired deletion and length-neutral padding to separate content value from
  context cost, and estimates the resulting values with a rollout-budgeted
  estimator for noisy agent evaluations. 
  On four agentic benchmarks, we assess the \emph{faithfulness},
  \emph{actionability}, and \emph{explanation} of \textsc{SkillSV}: it
  recovers unit interactions, preserves aggregate skill lift, and guides safe
  pruning and compression.
\end{abstract}


\section{Introduction}
\label{sec:intro}

Agent skills are reusable artifacts that turn LLM reasoning into concrete procedures~\cite{ding2026agent}. Automated optimizers increasingly write them, proposing and accepting edits under execution feedback~\cite{agrawal2026gepa,skillopt,trace2skill,yuksekgonul2024textgrad,yang2026autoskill,alzubi2026evoskill}.
These loops raise aggregate scores, but return one long skill whose credit
assignment is a black box. What is each piece of a skill worth?
 
We focus on intra-skill valuation, which attributes a skill's performance to its internal unit (e.g., a rule or a script).  This fine-grained perspective is crucial because skills are inherently optimized, debugged, and evolved at the granularity of editable units~\cite{ren2026self}. 
To make skill value well-defined, we fix three ingredients: a skill artifact under evaluation, a target agent, and a held-out task distribution with a metric.

\begin{figure}[tb!]
\centering
\includegraphics[width=\textwidth]{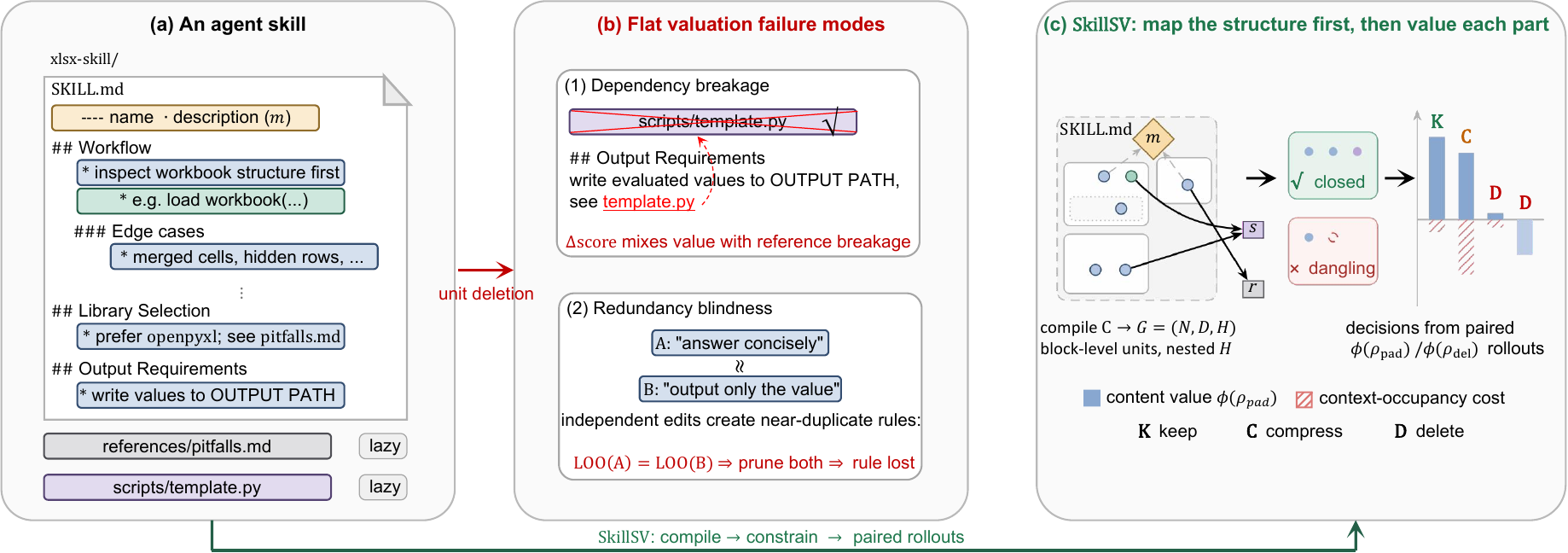}%
\caption{Why flat valuation misvalues agent skills, and how \textsc{SkillSV} fixes it.}
\label{fig:teaser}
\end{figure}

While existing literature successfully assigns credit to training data~\cite{ghorbani2019datashapley, wang2023databanzhaf,maleki2013bounding, lin2025lead} or attributes contributions to specific model parameters~\cite{ghorbani2020neuron,liu2023promptvaluation, mohammadi2024explaining,xu2026model,huang2026discovering}, these paradigms do not seamlessly transfer to agent skills.
A skill is not a singular data point or prompt, but a structured artifact that may combine Markdown hierarchy, triggers,  rules, examples, scripts, and auxiliary files.

This mismatch creates two challenges (Fig.~\ref{fig:teaser}).

\textbf{Challenge I: structural counterfactuals.}
A coalition of skill units is not automatically a valid skill. Removing one unit
may leave a dangling reference, an undefined symbol, a missing script, or a
section fragment detached from its lead-in. Scoring such broken artifacts would
measure the agent's robustness to malformed skills, not the value of the unit
removed. Skill valuation therefore needs a way to compile the artifact into
units, dependencies, and hierarchy, and to evaluate only counterfactuals that
still read like real skills.

\textbf{Challenge II: budgeted rollout estimation.}
Even valid counterfactuals are expensive to score. Each coalition must be
rendered as a skill, given to the agent, run on held-out tasks, and verified.
The estimates are noisy because task difficulty dominates many one-unit
differences. In addition, deleting a unit changes both the information available
to the agent and the prompt length, so a score change may reflect content,
context occupancy, or both.

\emph{Why existing recipes fail.}
These two challenges explain why flat ablation-style attribution is misleading for skills. The same issue also affects flat Shapley-style valuation: averaging
over arbitrary subsets gives credit for contexts that are not valid skill
artifacts~\cite{ghorbani2019datashapley,kwon2022beta}. Leave-one-out (LOO) is the simplest failure case~\cite{cook1977detection}.
If a surviving unit depends on a removed one, LOO measures breakage rather than
contribution. 

We propose \textsc{SkillSV}, a structure-aware valuation framework for agent 
skills. It compiles a skill into units, dependencies, and hierarchy; defines
Shapley-style unit values over feasible insertion orders; and estimates them
with paired deletion/padding rollouts under a fixed budget. Its output supports editing decisions:
  keep, compress, or delete.

\noindent\textbf{Contributions.}
(i) We formulate intra-skill valuation as a structure-constrained cooperative
game, identifying structural counterfactuals and budgeted rollout estimation as
the two challenges that distinguish skills from flat data valuation.
(ii) We introduce a deterministic compiler, paired deletion/padding
counterfactuals, and a chain-coupled estimator for noisy agent rollouts.
(iii) We show that \textsc{SkillSV} recovers interaction-sensitive unit values,
satisfies value closure, and guides safe pruning and compression with negligible performance loss across four benchmarks.

\section{Related Work}
\label{sec:related}
We first introduce key concepts foundational to our skill valuation. Appendix~\ref{app:related} provides an extended discussion and Appendix~\ref{app:notation} summarizes notation.

\subsection{Agent Skills \&  Automated Skill Optimization} 
Agent skills are reusable artifacts that store task-solving behavior, e.g., instructions, tool-use procedures, examples, scripts, and domain heuristics. They let language agents reuse prior experience across tasks instead of solving each instance from scratch~\cite{ding2026agent,ren2026self}.

Recent work has moved from manually authored skills toward automated skill generation and refinement. These methods optimize skills through execution  feedback~\cite{yuksekgonul2024textgrad,agrawal2026gepa,skillopt}, trajectory distillation~\cite{trace2skill,yang2026autoskill,zhou2026memento}, reinforcement learning~\cite{shi2026skill1,vishe2026skillr1}, and evolutionary search~\cite{alzubi2026evoskill,zhang2026coevoskills,ma2026skillclaw, ma2026skillgen}. Across these lines, the common goal is to improve the performance of the resulting skill as a whole. 
While existing methods improve overall skill performance, they rarely attribute these gains to specific internal components.

\subsection{Valuation via Cooperative Games}
The Shapley value offers a principled way to assign credit when several
components jointly determine a payoff. A (transferable-utility) cooperative game
is a pair $(\units,\val)$ with player set $\units=\{1,\dots,n\}$ and
characteristic function $\val:2^{\units}\!\to\!\mathbb{R}$, where $\val(S)$ is the
payoff a coalition $S\subseteq\units$ can jointly secure. The \emph{Shapley
value}~\citep{shapley1953value} is the unique allocation satisfying efficiency,
symmetry, linearity, and the null-player axiom; in permutation form,
\begin{equation}
\phi_i \;=\; \frac{1}{|\Pi(\units)|}\sum_{\pi\in\Pi(\units)}
\big[\val(\prefixset{i}{\pi}\cup\{i\})-\val(\prefixset{i}{\pi})\big],
\label{eq:shapley}
\end{equation}
where $\Pi(\units)$ is the set of $n!$ orderings of $\units$ and
$\prefixset{i}{\pi}=\{\,j\in\units:\pi(j)<\pi(i)\,\}$ collects players
preceding $i$ in $\pi$. Each summand is the \emph{marginal contribution} of $i$
to the coalition it joins, so averaging over all orderings allocates credit
fairly under interaction: redundant players split the value they duplicate, while
complementary players are rewarded for the synergy they unlock. Replacing the
uniform average over $\Pi(\units)$ with an arbitrary distribution over orderings
yields Weber's \emph{probabilistic values}~\citep{weber1988probabilistic}, which
we later specialize to structure-respecting skill orderings.

\subsection{Shapley-Based Valuation in ML}
Shapley-based valuation has been widely applied in machine learning, either to
internal components or to training data. Component-level methods score neurons~\cite{ghorbani2020neuron},
prompts~\cite{liu2023promptvaluation}, model decisions~\cite{mohammadi2024explaining},
parameters~\cite{xu2026model}, and mixture-of-experts experts~\cite{huang2026discovering},
while data-level methods estimate the worth of individual training examples~\cite{ghorbani2019datashapley,kwon2022beta,wang2023databanzhaf,maleki2013bounding}.
Across both lines, the units being valued are flat and unstructured, so uniform
sampling over coalitions suffices. In contrast, skill components are
interdependent, which calls for structure-respecting orderings rather than
uniform ones.

\section{Method}
\label{sec:method}

\subsection{Skill Compilation}
\label{sec:compilation}
Addressing \emph{structural counterfactuals}, we build counterfactual skills by compiling the skill artifact into a structured game.
The compiler $\comp$ maps a skill to a triple $\game=(\units,\dep,\hier)$ with no model in the loop, where $\units$ are valuation units, $\dep$ records dependency constraints among units, and $\hier$ records the document hierarchy used to constrain insertion orders.

\paragraph{Units $\units$.}
We cut valuation units at the granularity at which skill optimizers edit the artifact~\citep{skillopt,agrawal2026gepa}. In the Markdown body, each top-level list item opens a unit (e.g., $n_1$ in Fig.~\ref{fig:skillgraph}),  while indented continuations stay with their parent. Headings are structural scaffolding rather than players: a heading is rendered whenever any unit under it survives. Auxiliary files, such as scripts and references, are compiled as resource units (shown in Fig.~\ref{fig:skillgraph}).

\paragraph{Dependencies $\dep$.}
Dependencies implement the structural constraint introduced in Sec.~\ref{sec:intro}: a coalition is evaluated only if the units it keeps still form a skill that reads like a valid artifact. An edge of $\dep$ points from a unit to a unit it needs; feasible coalitions must therefore be downward closed under these edges. We extract nine families of edges.
Three representative ones are as follows. \emph{Reference edges} cover the syntactically visible cases of markdown links, file-path tokens, and verbatim heading mentions e.g., $n_5 \rightarrow r$ in Fig.~\ref{fig:skillgraph}. \emph{Semantic edges} bind a symbol's definition to its use. They carry no syntactic marker and are recovered by analysing the code blocks themselves~e.g., $n_4 \rightarrow s$.  \emph{Continuity edges} bind fragments that a patch-editing history left adjacent, such as an orphaned table body and the unit carrying its header. e.g., $\{n_1 \rightarrow m, n_4 \rightarrow m\}$.
 
\begin{figure}[t]
  \centering
  \begin{minipage}[t]{0.49\textwidth}
    \centering
    \includegraphics[width=\linewidth]{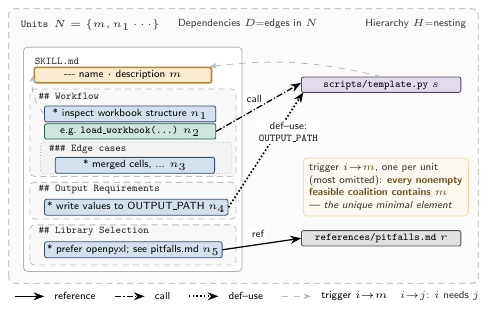}
    \caption{Compiled skill graph $\game=(\units,\dep,\hier)$.}
    \label{fig:skillgraph}
  \end{minipage}\hfill
  \begin{minipage}[t]{0.49\textwidth}
    \centering
    \includegraphics[width=\linewidth]{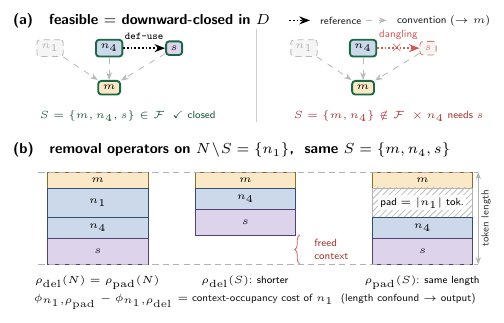}
    \caption{Feasible prefixes and removal operators.}
    \label{fig:operator}
  \end{minipage}
\end{figure}

\paragraph{Hierarchy $\hier$.}
The hierarchy $\hier$ represents the document tree (from directories down to valuation units). While dependencies $\dep$ dictate \emph{which} unit coalitions are valid, $\hier$ constrains the \emph{order} in which they are evaluated. Relying solely on $\dep$ allows unnatural, interleaved evaluation orders (e.g., $m, n_1, n_4, n_2, n_3$), where a unit might be judged against fragmented, partially opened sections. To prevent this, $\hier$ enforces \emph{contiguity}: sibling units must be grouped and evaluated together as a block (e.g., $m, [n_1, n_2, n_3], n_4$). This guarantees that units are evaluated against complete, intact sections.
Because dependencies cannot enforce contiguity, $\hier$ restricts evaluation permutations \citep{winter1989value}, while $\comp$ repairs structural conflicts between $\hier$ and $\dep$ during sampling.
   
\paragraph{Feasibility $\feas$.}
Given $\units$ and $\dep$, $\feas$ is the family of coalitions we allow the renderer to evaluate. A subset $S\subseteq\units$ is feasible if it is downward  closed under dependencies: whenever $u\in S$ and $u\to v\in\dep$, then $v\in S$ (in~Fig~\ref{fig:operator}(a)). Thus a kept unit must also keep the units it needs. The purpose is distributional, not only syntactic. Feasible coalitions should  still read like real skills: no removed scripts referenced by surviving text, no  undefined symbols, and no fragments detached from their lead-in. We do not fix infeasible coalitions by rewriting the remaining text, because that would  measure deletion plus rewriting rather than unit value .

Appendix~\ref{app:compilation} lists the full rule table and a walk-through.


\subsection{From Skill Structure to Unit Value}
\label{sec:game}
The compiled structure tells us which pieces of a skill may coexist, but not how valuable any piece is. We assign value by an insertion thought experiment: place the units in a valid order, add them one at a time, and measure how much the agent's score changes when a given unit appears. The value of a unit is the average of this \emph{marginal gain} over valid partial-skill contexts.

Two choices are specific to skills. First, not every subset is a valid partial skill: removing a target, a definition, or part of a section can leave a broken  document. This yields feasible insertion orders and their distribution $\mu$.  Second, a valid subset is still only a set of units, not text. Rendering it requires a removal operator $\rho$, whose choice determines whether removed content also shortens the prompt.

\paragraph{Feasible orders and $\mu$.}
A unit is valued by the score change when it is inserted into a partial skill.  We therefore only use insertion orders whose every prefix is itself a valid  partial skill. Formally, a permutation of $\units$ is feasible if each prefix lies in $\feas$ and is contiguous under $\hier$. Fig.~\ref{fig:operator}(a) shows the coalition-level condition behind this definition: each prefix must be downward-closed in $\dep$, so a unit is never evaluated without the units it  needs.

Many feasible orders may remain. Different samplers weight different partial skills, so we make this choice by indexing the value by a distribution $\mu$ over feasible orders. We do not claim a canonical uniform value: the layerwise sampler used in practice is generally not uniform over linear extensions, and uniform sampling is computationally hard.
Appendix~\ref{app:sampler} gives the sampler and its support properties.

\paragraph{Rendering coalitions with $\rho$.}
An insertion order tells us which coalitions to score, but a coalition $S\subseteq\units$ is not yet a prompt. A removal operator $\rho$ renders $S$ as a skill by keeping the units in $S$ and handling the missing units in $\units\setminus S$.

We require $\rho$ to be local: it may change only removed units, while surviving units remain byte-for-byte unchanged. This makes the measured effect a deletion
effect, not deletion plus rewriting. Length is the second issue. Removing a unit removes information and also shortens the prompt, so we use two local operators. $\rhodel$ deletes missing units outright; $\rhopad$ replaces them with length-matched neutral placeholders. Fig.~\ref{fig:operator}(b) illustrates the two renderings on the same feasible coalition.

\paragraph{Unit value.}
 Given a feasible-order distribution $\mu$ and a rendering operator $\rho$, the value of an execution unit $i$ is its expected marginal contribution. This is the random-order form of the Shapley marginal, with two changes: orders are restricted to feasible skill orders, and the uniform distribution is replaced   by the declared distribution $\mu$~\citep{ghorbani2019datashapley,chi2025pcwinter}:
  \begin{equation}
  \label{eq:value}
  \phi_{i,\rho}(\mu)=\E_{\pi\sim\mu}\!\left[
    \val_\rho\!\big(\prefixset{i}{\pi}\cup\{i\}\big)-
    \val_\rho\!\big(\prefixset{i}{\pi}\big)\right],
  \end{equation}
  where $\prefixset{i}{\pi}$ is the set of units preceding $i$ in $\pi$.
  The coalition value is the average verified rollout score after the agent is
  given the rendered skill $\rho(S)$:
  \begin{equation}
  \label{eq:coalition-value}
  \val_\rho(S)=\frac{1}{M}\sum_{t=1}^{M}
  \score\!\left(\agentop(\rho(S),t)\right),
  \end{equation}
  where $M$ is the number of held-out tasks, $\agentop(\rho(S),t)$ is the agent's
  rollout on task $t$ using the rendered skill $\rho(S)$, and $\score$ is the
  programmatic task verifier. For stochastic agents, $\score$ denotes the rollout
  average used by the evaluator. 
 
  With this definition, the two renderings give complementary views of the same
  unit. $\rhodel$ measures the net effect of removing the unit from the prompt,
  while $\rhopad$ measures its marginal contribution with the document footprint
  held fixed. We therefore report: 
  \begin{align*}
  \text{content value} \;&=\; \phi_{i,\rhopad}, \\
  \text{context cost}  \;&=\; \phi_{i,\rhopad}-\phi_{i,\rhodel}, \\
  \text{net effect}    \;&=\; \phi_{i,\rhodel}.
  \end{align*}
  A unit with positive content value but large context cost is a compression target rather than a deletion target. At the grand coalition the two renderings
   agree, $\val_{\rhodel}(\units)=\val_{\rhopad}(\units)$, which gives a runtime check on the renderer. That is, the context-cost term is interpreted under the placeholder-neutrality assumption.
  
  The frontmatter unit $m$ is treated separately because it triggers the skill
  rather than executing one part of it. We define its trigger value as
  \[
  \theta_m=\val_{\rhodel}(\{m\})-\val_{\rhodel}(\varnothing).
  \]
  The deletion operator is used here because $\rhopad(\varnothing)$ would be a
  document of placeholders rather than a bare agent. Since every execution unit
  depends on $m$, every feasible order begins with $m$; all other units are
  therefore valued conditional on the skill being triggered

  \begin{algorithm}[t]
  \caption{Chain-coupled task-window estimation.}
  \label{alg:estimator}
  \begin{algorithmic}[1]
  \REQUIRE compiled game $\game=(\units,\dep,\hier)$, stratified task panel,
    number of orders $K$, window size $b$, tolerance $\tau$
  \ENSURE trigger value $\hat\theta_m$ and unit values
    $\hat\phi_{i,\rhodel},\hat\phi_{i,\rhopad}$ for execution units $i\neq m$
  \STATE initialize running sums for $\hat\theta_m$ and all
    $\hat\phi_{i,\rho}$ to zero
  \FOR{$k=1$ \TO $K$}
    \STATE draw a feasible order $\pi_k$ respecting $\dep$ and $\hier$
    \STATE draw a stratified task window $B_k$ of size $b$, independent of $\pi_k$
    \STATE evaluate $\bar v_{\rhodel}(\varnothing,B_k)$ and
      $\bar v_{\rhodel}(\{m\},B_k)$
    \STATE add
      $\bar v_{\rhodel}(\{m\},B_k)-\bar v_{\rhodel}(\varnothing,B_k)$
      to the running sum for $\hat\theta_m$
    \STATE evaluate the full-skill anchor $\bar v(\units,B_k)$
      \COMMENT{same rendered skill for both operators}
    \FOR{$\rho\in\{\rhodel,\rhopad\}$}
      \STATE $S\leftarrow\{m\}$
      \STATE evaluate $\mathrm{prev}\leftarrow \bar v_\rho(S,B_k)$
        \COMMENT{reuse the value above when $\rho=\rhodel$}
      \FOR{each execution unit $i$ after $m$ in $\pi_k$}
        \STATE $S\leftarrow S\cup\{i\}$
        \STATE evaluate $\mathrm{cur}\leftarrow \bar v_\rho(S,B_k)$
        \STATE add $\mathrm{cur}-\mathrm{prev}$ to the running sum for
          $\hat\phi_{i,\rho}$
        \STATE $\mathrm{prev}\leftarrow\mathrm{cur}$ \\
        \COMMENT{truncate only when the tolerance holds for both operators}
        \IF{$|\mathrm{cur}-\bar v(\units,B_k)|\le\tau$}
          \STATE stop this chain for operator $\rho$
          \STATE \textbf{break}
        \ENDIF
      \ENDFOR
    \ENDFOR
  \ENDFOR
  \STATE divide every running sum by $K$
  \RETURN $\hat\theta_m$, $\hat\phi_{i,\rhodel}$, $\hat\phi_{i,\rhopad}$
  \end{algorithmic}
  \end{algorithm}

\subsection{Estimate Unit Values under a Rollout Budget}
\label{sec:estimation}
Addressing \emph{budgeted rollout estimation}, this section estimates the value in Eq.~\ref{eq:value} under a fixed rollout budget.
The difficulty is that no coalition value is a lookup: obtaining one requires re-rendering the skill, running the frozen agent on held-out tasks, and verifying every outcome.
A full two-operator audit would require
\[
  R_{\mathrm{full}}
  =
  \underbrace{2}_{\text{operators}} \times 
  \underbrace{K}_{\text{orders}} \times 
  \underbrace{(n+1)}_{\text{prefix chain}} \times 
  \underbrace{M}_{\text{tasks}}
\]
agent rollouts, which is prohibitive for nontrivial skills.
For $n=50$ units, $M=40$ tasks and only $K=10$ orders, already $2\times10\times51\times40=40{,}800$ agent rollouts.
This makes naive estimation impractical.

Two observations make this reduction possible. First, a sampled feasible order
gives a whole prefix chain. For an order $\pi_k$, let $S_j$ be the set of the
first $j$ units in that order, so
\[
\varnothing=S_0\subset S_1\subset\cdots\subset S_n=\units .
\]
Each adjacent pair gives the marginal contribution of the unit added at that
step. Thus the $n+1$ coalition evaluations along one order yield one marginal
sample for every unit, rather than only for a single unit.

Second, the task panel can be windowed only if the window is shared along the
chain. We draw a stratified task window $B_k$ of size $b\ll M$ for each sampled
order and evaluate every prefix of that order on the same $B_k$. Let
$\bar v_\rho(S,B_k)$ denote the average verified rollout score of coalition
$S$ on the task window $B_k$ under rendering operator $\rho$:
\[
\bar v_\rho(S,B_k)=\frac{1}{b}\sum_{t\in B_k}
\score\!\left(\agentop(\rho(S),t)\right).
\]
The marginal at step $j$ is then the paired difference
\[
\bar v_\rho(S_j,B_k)-\bar v_\rho(S_{j-1},B_k).
\]
Using the same tasks on both sides cancels much of the task-level difficulty
variation; using different windows would mix the unit effect with differences
between tasks. This pairing makes small windows usable, so the saved budget can
be spent on sampling more feasible orders.

 We call this mechanism \emph{chain-coupled task windows}. The resulting
  estimator over $K$ sampled feasible orders is
  \begin{equation}
  \label{eq:estimator}
  \hat\phi_{i,\rho}
  =
  \frac{1}{K}\sum_{k=1}^{K}
  \left[
  \bar v_\rho\big(\prefixset{i}{\pi_k}\cup\{i\},B_k\big)
  -
  \bar v_\rho\big(\prefixset{i}{\pi_k},B_k\big)
  \right].
  \end{equation}
  Because each window is drawn independently of the sampled order and gives each
  task equal inclusion probability, Eq.~\ref{eq:estimator} targets the
  panel-level value in Eq.~\ref{eq:value}. The implementation enforces pairing by
  requiring the two sides of every marginal difference to use the identical task
  list. 

  Windowing reduces the task factor, but a long chain can still spend rollouts on
  prefixes whose score has already matched the full skill. We therefore use
  noise-gated truncation~\cite{ghorbani2019datashapley,kwon2022beta}. Each chain first evaluates the anchor
  $\bar v_\rho(\units,B_k)$. While scanning prefixes, if
  \[
  \left|\bar v_\rho(S_j,B_k)-\bar v_\rho(\units,B_k)\right|\le \tau,
  \]
  the remaining units on that chain receive zero marginal contribution. The
  threshold $\tau$ controls truncation bias by bounding the total residual
  contribution left in the truncated suffix on that window. Since the test uses a
  noisy window mean, $b$ and $\tau$ are calibrated jointly. The trigger unit $m$
  is minimal in every feasible order, so its trigger value is evaluated before
  truncation can occur.


 Algorithm~\ref{alg:estimator} also gives the actual cost. Per sampled order,
  we evaluate four anchors: $\bar v_{\rhodel}(\varnothing,B_k)$,
  $\bar v_{\rhodel}(\{m\},B_k)$, $\bar v_{\rhopad}(\{m\},B_k)$, and the
  full-skill anchor $\bar v(\units,B_k)$ shared by both operators. The remaining
  cost is two truncated prefix walks, one for $\rhodel$ and one for $\rhopad$.
  Since the full-skill value is already cached as the anchor, the last prefix need
  not be evaluated again. If $\gamma\in(0,1]$ is the average fraction of
  intermediate prefixes that survive the truncation gate, the rollout cost is
  \[
  R
  =
  K\,b\,\big[4+2\gamma(n-2)\big],
  \qquad
  \frac{R}{R_{\mathrm{full}}}
  \approx
  \gamma\,\frac{b}{M}.
  \]
  For $n=50$, $M=40$, $K=10$, $b=8$, and a measured
  $\gamma\approx0.2$, this gives a more than tenfold reduction. 
  
Appendix~\ref{app:estimator} gives the unbiasedness and variance analysis.

\section{Experiments}
\label{sec:experiments}

We evaluate whether \textsc{SkillSV} produces values that are faithful to the defined game, useful for editing skills, and explanatory of optimized skill structure. The experiments answer three questions:
\emph{(i) Faithfulness:} does \textsc{SkillSV} recover known unit values and
preserve the aggregate value of real skills?
\emph{(ii) Actionability:} can the values guide pruning and compression without
meaningful held-out performance loss?
\emph{(iii) Explanation:} what attribution patterns in optimized skills make
such compression possible?

\subsection{Experimental Setup}
  
\textbf{Dataset.} 
We evaluate \textsc{SkillSV} on four agentic benchmarks covering different task types and interaction budgets: {LiveMath}~\citep{he2026livemath},
{OfficeQA}~\citep{opsahlong2026officeqa}, {SpreadsheetBench(SSB, in short)}~\citep{ma2024spreadsheetbench}, and {ALFWorld}~\citep{shridhar2021alfworld}. The details of the four 
benchmarks are summarized in Table~\ref{tab:anchors}.  
For each benchmark, we evaluate the best-performing \texttt{skill.md} generated by four automatic skill optimizers: Trace2Skill~\citep{trace2skill},
TextGrad~\citep{yuksekgonul2024textgrad}, GEPA~\citep{agrawal2026gepa}, and SkillOpt~\citep{skillopt}. 
Note that, \textsc{SkillSV} is optimizer-agnostic: it values only the compiled skill document, independent of how the skill was produced.

\begin{table}[tb!]
\centering
\caption{Benchmark statistics and score anchors (in \%). $V(\varnothing)$:
target model with no skill; $V(\{m\})$: metadata/trigger unit only;
$V(N)$: full skill. Total lift $=V(N)-V(\varnothing)$, brackets report the 95\% bootstrap
  confidence interval (CI); content lift $=V(N)-V(\{m\})$.}
\label{tab:anchors}
\small
\setlength{\tabcolsep}{8pt}
\renewcommand{\arraystretch}{1.25}
\begin{tabular}{@{}lcccccc@{}}
\toprule
\textbf{Benchmark} & \makecell{Max\\turns} & $V(\varnothing)$ & $V(\{m\})$ & $V(N)$ & \makecell{Total\\lift} & \makecell{Content\\lift} \\
\midrule
LiveMath    &  1 & 48.3 & 48.3 & 65.0 & \gap{+16.7}{+3.3}{+30.0}   & $+16.7$ \\
OfficeQA    & 24 & 51.2 & 53.5 & 69.8 & \gap{+18.6}{+10.5}{+27.9}  & $+16.3$ \\
Spreadsheet & 30 & 27.5 & 28.7 & 78.7 & \gap{+51.2}{+40.0}{+62.5}  & $+50.0$ \\
ALFWorld    & 50 & 88.3 & 86.7 & 96.7 & \gap{+8.3}{+1.7}{+16.7}    & $+10.0$ \\
\bottomrule
\end{tabular}
\end{table}

\noindent \textbf{Baselines.}
We compare \textsc{SkillSV} with three baselines that rank the same compiled 
units in Sec.~\ref{sec:method}: Closure-LOO~\cite{cook1977detection}, an LLM judge~\cite{zheng2023judging}, 
and random ranking. Closure-LOO ablates each unit with its dependency closure, 
while the LLM judge with GPT-5.5 scores units independently. 
All methods use the same held-out pruning protocol.

\noindent \textbf{Implementation details.}
We compile skills into units and evaluate coalitions with a frozen target agent
(GPT-5.5) under each benchmark's native metric. unit values are estimated with
 deletion ($\rhodel$) and length-neutral padding ($\rhopad$), which separate
content contribution from context-occupancy cost. We use a budgeted estimator
($K{=}12$, $b{=}8$, $\tau{=}0.05$) and a 1000-replicate task-level bootstrap.
Full experimental details are given in  Appendix~\ref{app:setting}.

\begin{table}[tb!]
\centering
\caption{Value-closure check on four benchmarks. $\Lambda=V(N)-V(\{m\})$ is the content lift (95\% CIs in brackets); \cmark/\xmark\ marks whether the estimator's CI covers $\Lambda$. \textsc{SkillSV} totals recover $\Lambda$, whereas LOO is often miscalibrated. Note that, the inconsistency between  $V(N)$ here and in Table~\ref{tab:anchors}, because we evaluated it on the disjoint panel of held outs for fair comparison.}
\label{tab:value-closure}
\footnotesize
\setlength{\tabcolsep}{4pt}
\renewcommand{\arraystretch}{1.2}
\fitwidth{%
\begin{tabular}{@{}l cc ccc c cc cc@{}}
\toprule
& \multicolumn{2}{c}{\textbf{Size}} & \multicolumn{3}{c}{\textbf{Coalition value}}
& \textbf{Lift} & \multicolumn{2}{c}{\textbf{\textsc{SkillSV}}} & \multicolumn{2}{c}{\textbf{Closure-LOO}}\\
\cmidrule(lr){2-3}\cmidrule(lr){4-6}\cmidrule(lr){7-7}\cmidrule(lr){8-9}\cmidrule(lr){10-11}
\textbf{Benchmark} & Units & Tasks & $V(\varnothing)$ & $V(\{m\})$ & $V(N)$
& $\Lambda$ & $\Sigma\phi$ & $\tfrac{\Sigma\phi}{\Lambda}$
& $\Sigma\mathrm{LOO}$ & $\tfrac{\Sigma\mathrm{LOO}}{\Lambda}$\\
\midrule
LiveMath & 16 & 60 & 0.483 & 0.483 & 0.633
& \gap{+0.150}{+0.033}{+0.283} & \gap{+0.146}{+0.073}{+0.229} & $0.97\times$~\cmark
& \gap{-0.125}{-1.250}{+0.875} & $-0.83\times$~\xmark\\
OfficeQA & 25 & 86 & 0.488 & 0.535 & 0.733
& \gap{+0.198}{+0.116}{+0.291} & \gap{+0.188}{+0.135}{+0.240} & $0.95\times$~\cmark
& \gap{+0.875}{+0.250}{+1.625} & $4.43\times$~\xmark\\
Spreadsheet & 49 & 80 & 0.275 & 0.275 & 0.775
& \gap{+0.500}{+0.388}{+0.613} & \gap{+0.479}{+0.396}{+0.573} & $0.96\times$~\cmark
& \gap{-0.500}{-1.500}{+0.500} & $-1.00\times$~\cmark\\
ALFWorld & 48 & 60 & 0.883 & 0.900 & 0.950
& \gap{+0.050}{+0.000}{+0.117} & \gap{+0.052}{+0.021}{+0.083} & $1.04\times$~\cmark
& \gap{-0.500}{-1.125}{+0.000} & $-10.00\times$~\xmark\\
\bottomrule
\end{tabular}}
\end{table}

\subsection{\textsc{SkillSV} Measures the Right Value}

\noindent\textbf{Recovering planted and real interactions.}
We test faithfulness in two settings. First, we use synthetic skills where the
  ground-truth interaction pattern is known: a redundant OR pair, a complementary
  AND pair, a harmful unit, and a dead unit. Second, we use real optimized skills of Spreadsheet and check whether the estimated values agree with direct interaction tests,
  such as removing one unit versus removing both units in a redundant pair.

  Fig.~\ref{fig:mechanism}(a) shows that \textsc{SkillSV} recovers the planted
  roles across all synthetic cases. The baselines fail predictably: Closure-LOO assigns zero to the redundant pair and double-counts the complementary pair, the LLM judge is inconsistent, and random provides no meaningful signal.

  This effect extends to real skills (Fig.~\ref{fig:mechanism}(b)). For a matched redundant pair, removing either unit alone barely affects the score, whereas removing both causes a clear drop.  \textsc{SkillSV} correctly recover this. 
  The interaction value is real, yet entirely missed by single-removal evaluations.

\begin{figure}[tb!]
\centering
  \begin{minipage}[t]{0.535\textwidth}
    \centering
    \includegraphics[width=\linewidth]{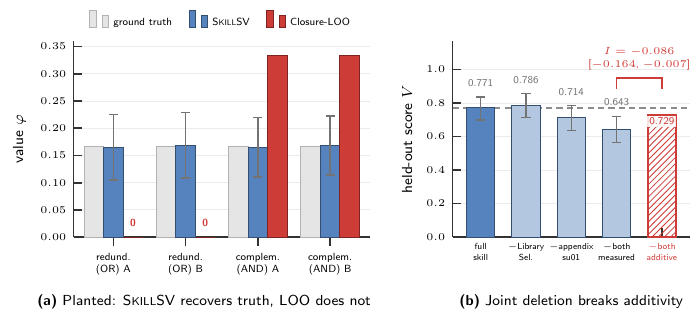}
    \caption{\textsc{SkillSV} recovers planted OR/AND unit values.}
    \label{fig:mechanism}
  \end{minipage}\hfill
  \begin{minipage}[t]{0.445\textwidth}
    \centering
    \includegraphics[width=\linewidth]{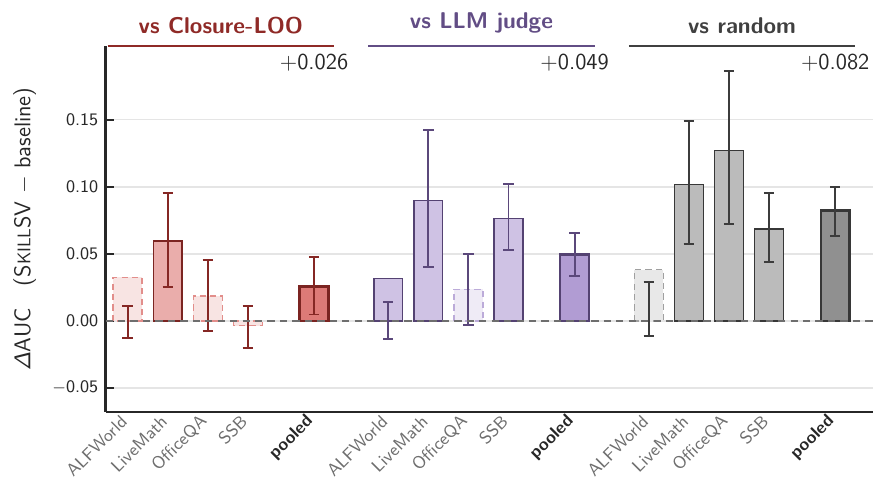}
    \caption{AUC summary of the pruning curves.}
    \label{fig:prune_auc}
  \end{minipage}
\end{figure}

\noindent \textbf{Recovering total value.}
A second faithfulness check asks whether the unit scores add up to the observed
  gain from the skill content. For each real skill, we first measure three anchor
  values: the bare agent $V(\varnothing)$, the trigger-only skill $V(\{m\})$, and
  the full skill $V(N)$. The content lift is
  \[
  \Lambda = V(N)-V(\{m\}),
  \]
  the improvement attributable to execution units beyond the frontmatter. If an
  attribution method assigns calibrated unit values, then summing those unit
  values should recover $\Lambda$.

  Table~\ref{tab:value-closure} shows that \textsc{SkillSV} passes this closure
  check. Across all four benchmarks, $\sum_i \phi_i$ closely matches the measured
  content lift, with ratios from $0.95\times$ to $1.04\times$. This is the \emph{empirical counterpart of Shapley efficiency}: the budgeted estimator preserves
  the aggregate value of the skill content. Closure-LOO fails the same check, often with the wrong sign or totals several times larger than $\Lambda$,  because single-deletion effects are not additive in the presence of redundancy and complementarity.

\subsection{SkillSV Enables Safe Compression}

\noindent\textbf{Efficacy of Safe Compression.}
To directly evaluate the decision-making utility of unit valuations, we conduct a pruning test: units are sequentially removed in ascending order of their estimated values, and performance is tracked on a held-out panel. An ideal valuation accurately isolates redundant content, preserving performance during early pruning steps and yielding a larger Area Under the Curve (AUC).

As shown in Fig.~\ref{fig:prune_auc}, results indicate that \textsc{SkillSV} most robustly guides ``safe compression.'' In the pooled evaluation across all four benchmarks, \textsc{SkillSV} achieves a significantly higher AUC than all baselines (+0.026 vs. LOO, +0.049 vs. LLM judge, +0.082 vs. random, with no 95\% CIs containing zero), showing that the low-value content it identifies can indeed be safely discarded.

\begin{figure}[tb!]
\centering
\includegraphics[width=\textwidth]{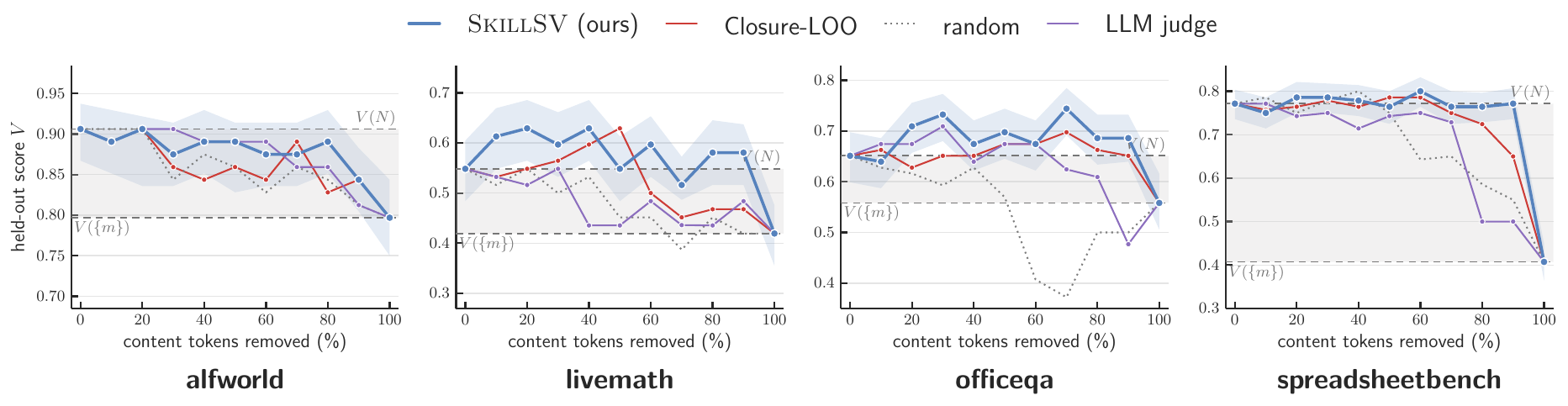}%
\caption{Held-out pruning curves comparing \textsc{SkillSV} with Closure-LOO, an LLM judge, and random unit deletion.}
\label{fig:prune_curve}
\end{figure}

\noindent\textbf{Robustness at High Pruning Ratios.}
Moreover, Fig.~\ref{fig:prune_curve} reports the full held-out pruning curves, where content units are removed from lowest to highest estimated value and each intermediate skill is re-scored. 

An ideal valuation should maintain the full-skill score $V(N)$ during early pruning, dropping toward the minimal-skill floor $V(\{m\})$ only after all expendable content is removed. Across all four benchmarks, \textsc{SkillSV} sustains peak performance at significantly higher pruning ratios than Closure-LOO, the LLM judge, or random ranking (which degrades rapidly and even collapses below $V(\{m\})$ on OfficeQA). 
This indicates that \textsc{SkillSV} ranks units by their true removable contribution rather than by isolated leave-one-out effects, enabling more aggressive compression before performance is affected.

\noindent\textbf{Lossless Skill Compression.}  Hard deletion tests whether a valuation can rank units for removal, but editing
  a skill in practice is less binary: a low-value unit may be deleted, merged with
  a redundant unit, or rewritten more compactly. We therefore run a single
  attribution-guided refinement step. The editor receives the original skill and
  the \textsc{SkillSV} report computed on panel A, including each unit's content
  value, context cost, and net effect. It is instructed to preserve high-value
  content, remove harmful or near-zero units, and compress units whose content
  value is positive but whose context cost is large.  

  We then evaluate the revised skill once on the disjoint panel B and compare it
  with the original skill on the same tasks. Table~\ref{tab:refine-prune} shows
  that the revised skills retain only 69\% of the original tokens on average, yet no significant performance change on all 4 benchmarks.
   Thus, \textsc{SkillSV} effectively facilitates the lossless compression of converged skills.  

\begin{table}[tb!]
\centering
\caption{Safe compression using \textsc{SkillSV}-only refinement.}
\label{tab:refine-prune}
\small
\setlength{\tabcolsep}{9pt}
\renewcommand{\arraystretch}{1.25}
\begin{tabular}{@{}lrrrrl@{}}
\toprule
\textbf{Benchmark} & $M$ & Tok. & $V_{\text{old}}$ & $V_{\text{new}}$ & $\Delta$ {\scriptsize[95\% CI]} \\
\midrule
LiveMath    &  62 & 96\% & 61.3 & 56.5 & $-4.8$ {\scriptsize$[-14.5,+4.8]$} \\
OfficeQA    &  86 & 80\% & 64.0 & 65.1 & $+1.2$ {\scriptsize$[-7.0,+9.3]$}  \\
Spreadsheet & 140 & 57\% & 73.6 & 73.6 & $+0.0$ {\scriptsize$[-5.0,+5.0]$}  \\
ALFWorld    &  64 & 42\% & 89.1 & 87.5 & $-1.6$ {\scriptsize$[-9.4,+4.7]$}  \\
\cmidrule(l){2-6}
\emph{mean} &     & \emph{69\%} & & & \emph{$-1.3$} \\
\bottomrule
\end{tabular}
\end{table}

\subsection{Why Compression Works}
The compression results above are possible because optimized skill files contain
many units whose measured contribution is small. 
We next give a deep analysis.

First, within that content, value is far from uniform. 
From Fig.~\ref{fig:concentration}(a), the top 10\% of units account for
21\% of the mass on OfficeQA, 35\% on LiveMath, 60\% on Spreadsheet, and
100\% on ALFWorld.
 This concentration explains why low-value pruning can remove
substantial text before performance changes.

Second, 
Fig.~\ref{fig:concentration}(b) illustrates the same point at unit level on
Spreadsheet. \textsc{SkillSV} separates a small number of influential
units from values near the estimator's noise floor. By contrast, Closure-LOO collapses estimates to zero because redundant units mask each other in the full context. True sparsity emerges only through multi-context evaluation.

\begin{figure}[tb!]
\centering
\includegraphics[width=0.7\textwidth]{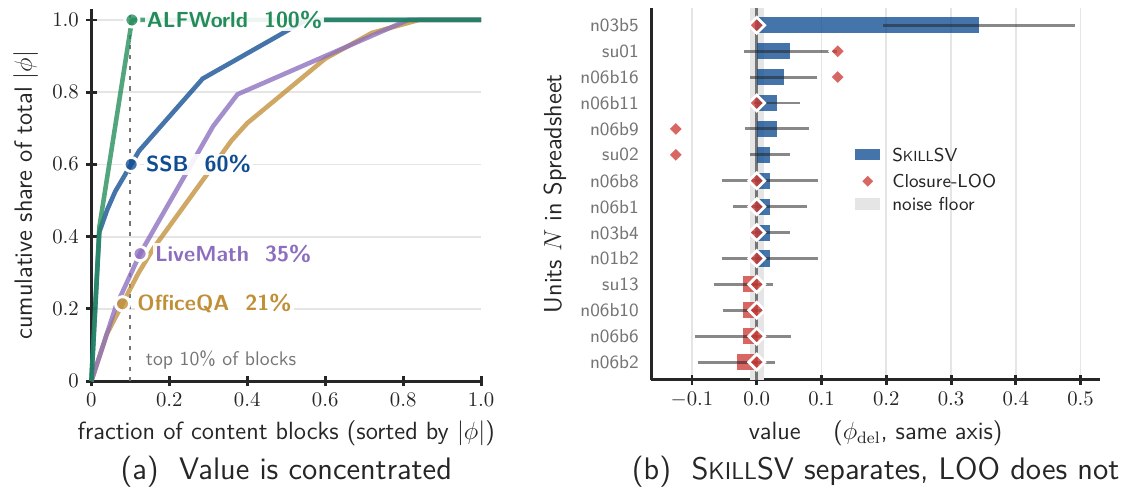}%
\caption{Concentration and resolution of skill values.}
\label{fig:concentration}
\end{figure}

\section{Conclusion}
\label{sec:conclusion}

We introduced \textsc{SkillSV}, a framework for attributing the performance of
  an agent skill to the units inside the skill artifact. The central point is that
  skill valuation is not flat ablation: counterfactual skills must respect
  dependencies and hierarchy, deletion must be separated from prompt-length
  effects, and estimates must be obtained under a tight rollout budget.
  \textsc{SkillSV} addresses these requirements by compiling skills into
  structured games, valuing units over feasible insertion orders, and estimating
  paired deletion/padding counterfactuals with chain-coupled task windows.

  Empirically, \textsc{SkillSV} recovers interaction-sensitive unit values,
  preserves aggregate skill lift, and supports safe pruning and compression across
  four agentic benchmarks. These results suggest that optimized skills often
  contain redundant or low-value content that can be identified without treating
  the skill as an opaque prompt. More broadly, structure-aware valuation provides
  a practical diagnostic layer for future skill optimization systems: not only
  whether a skill works, but which parts work, which parts are costly, and which
  parts can be safely changed.

\bibliographystyle{unsrtnat}
\bibliography{references}  

\clearpage
\appendix

\section{Extended Related Work}
  \label{app:related}

The main related-work section introduces the concepts needed to read the method.
Here we expand on two points: how \textsc{SkillSV} relates to cooperative games
with restricted cooperation, and how it differs from prior Shapley-based
valuation in machine learning.

\subsection{Cooperative Games with Restricted Cooperation}
\label{app:related-coop}
 Classical Shapley valuation averages marginal contributions uniformly over all
  player permutations, as in Eq.~\ref{eq:shapley}. Games with restricted
  cooperation retain this marginal-contribution perspective but constrain which
  coalitions or joining orders are admissible, or change how admissible orders are
  weighted.

Precedence-constrained games restrict admissible orders to linear
  extensions of a partial order~\citep{faigle1992shapley}. Consequently, every
  prefix is downward closed: a player can enter only after its prerequisites.
  This corresponds to the role of the dependency DAG $\dep$ in
  \textsc{SkillSV}. Winter's levels-structure value instead imposes a nested
  organization on players and requires members of the same hierarchical block to remain contiguous in an admissible order~\citep{winter1989value}. This
  corresponds to the role of the document hierarchy $\hier$.

  \paragraph{Relation to PC-Winter.}
   PC-Winter~\citep{chi2025pcwinter} combines precedence and hierarchy constraints
  derived from a graph computation structure, and values nodes over the resulting
  admissible orders. \textsc{SkillSV} follows the same broad perspective but
  differs in three ways. First, $\dep$ and $\hier$ are independently compiled from
  skill dependencies and document structure, so their compatibility is not
  guaranteed. Second, a coalition must be rendered into an executable skill by
  $\rho$, making its payoff $V_\rho$ renderer-dependent. Third, because our
  sampler need not be uniform over feasible orders, we explicitly index the value
  by its induced distribution $\mu$, following the probabilistic-value
  perspective~\citep{weber1988probabilistic}.

  Thus, \textsc{SkillSV} is not simply PC-Winter applied to a different data
  modality. It adopts the general idea of restricting marginal-contribution
  orders, while adding independently compiled dependency and hierarchy
  structures, renderer-dependent counterfactuals, and an explicit feasible-order
  distribution. In the unconstrained special case with uniform $\mu$, the
  conditional execution-unit game reduces to classical Shapley valuation.

\subsection{Shapley-Based Valuation in Machine Learning}
\label{app:related-ml}
  Shapley-based and related cooperative-game methods have been used to value
  training data, model components, and prompt or context fragments. Many of
  these methods use a flat player set. Grouped and hierarchical variants also
  exist, but their structure typically changes the attribution granularity or
  aggregation rule rather than declaring a text coalition structurally invalid.
  \textsc{SkillSV} uses structure for the latter purpose: dependencies determine
  admissible coalitions, and hierarchy restricts admissible joining orders.

  \paragraph{Data and model-component valuation.}
  Data Shapley assigns credit to training examples through their marginal effects
  on model utility~\citep{ghorbani2019datashapley}. Beta Shapley modifies the
  coalition weighting to improve robustness and estimation
  ~\citep{kwon2022beta}, while Data Banzhaf provides a related
  cooperative-game alternative~\citep{wang2023databanzhaf}. The same perspective
  has been applied to model components, including neurons
  ~\citep{ghorbani2020neuron}, parameters~\citep{xu2026model}, and
  mixture-of-experts experts~\citep{huang2026discovering}. These methods capture
  redundancy and complementarity, but their interventions remain defined without
  introducing dependency-closed coalitions or hierarchy-constrained orders.

  \paragraph{Prompt and context attribution.}
  Text-based attribution is closer to skill valuation. Prior work values prompt
  segments or input attributes~\citep{liu2023promptvaluation,
  mohammadi2024explaining}, tokens~\citep{horovicz2024tokenshap,
  xiao2025tokenshapley}, context sentences~\citep{cohenwang2024contextcite},
  in-context demonstrations~\citep{xie2024demoshapley}, and source documents
  ~\citep{ye2025clustershapley}. These methods generally construct coalitions by
  deleting, masking, or selecting text fragments. However, they do not explicitly
  model a dependency relation under which removing one fragment can invalidate
  another, or a document hierarchy that constrains the contexts in which a
  fragment should be evaluated.

  Skills make these constraints consequential. A surviving unit may reference a
  removed script, use a deleted definition, or become detached from its structural
  lead-in. \textsc{SkillSV} therefore evaluates only dependency-closed coalitions
  and hierarchy-compatible orders. Its local renderer additionally leaves
  surviving units unchanged, avoiding counterfactuals that combine deletion with
  uncontrolled rewriting.

  \paragraph{Removal effects and agentic evaluation.}
  Deleting prompt text can change both its information content and its context
  footprint. Prior attribution methods typically include both effects in a single
  intervention. \textsc{SkillSV} instead pairs deletion with length-matched
  padding. Under the placeholder-neutrality assumption stated in
  Sec.~\ref{sec:game}, the two interventions support separate estimates of
  content value, context cost, and net effect. This interpretation is therefore
  conditional on the placeholder acting only through its context footprint.

  The utility oracle also differs. Context-attribution methods typically explain
  a prediction, output probability, or generated response, whereas skill
  valuation measures verified performance over a held-out task distribution after
  multi-step agent execution. Each coalition may therefore require multiple
  stochastic, tool-using rollouts. Our estimator builds on permutation sampling
  and truncation developed for Data Shapley
  ~\citep{ghorbani2019datashapley,maleki2013bounding}, but couples adjacent
  marginals through shared task windows to reduce variance from the agentic
  oracle.

  In summary, prompt attribution provides the closest textual intervention
  setting, while restricted-cooperation games provide the closest structural
  formalism. \textsc{SkillSV} connects the two by valuing editable text units only
  through counterfactuals that remain valid, executable skills.

\section{Notation}
\label{app:notation}
  Table~\ref{tab:notation} summarizes the notation in order of its use: skill
  compilation, structural counterfactuals, unit valuation, and budgeted
  estimation.

  \begin{table}[tb!]
  \centering
  \caption{Summary of notation.}
  \label{tab:notation}
  \small
  \renewcommand{\arraystretch}{1.2}
  \begin{tabular}{@{}l p{0.80\linewidth}@{}}
  \toprule
  \textbf{Symbol} & \textbf{Meaning} \\
  \midrule

  \multicolumn{2}{@{}l}{\emph{Skill structure and compilation}}\\
  $\comp$
    & Deterministic compiler mapping a skill artifact to $\game$ \\
  $\game=(\units,\dep,\hier)$
    & Compiled skill structure \\
  $\units,\ n$
    & Set of valuation units, including the trigger unit; $n=|\units|$ \\
  $\dep$
    & Dependency DAG; $u\to v$ means that unit $u$ requires unit $v$ \\
  $\hier$
    & Document hierarchy over files, sections, and valuation units \\
  $m$
    & Frontmatter/trigger unit required by every execution unit \\
  $i$
    & An execution unit, $i\in\units\setminus\{m\}$ \\

  \addlinespace
  \multicolumn{2}{@{}l}{\emph{Coalitions and counterfactuals}}\\
  $S$
    & Coalition of retained units, $S\subseteq\units$ \\
  $\feas$
    & Dependency-closed coalitions:
      $\{S\subseteq\units: u\in S,\ u\to v\in\dep \Rightarrow v\in S\}$ \\
  $\rho$
    & Local counterfactual rendering operator mapping $S$ to a skill artifact \\
  $\rhodel$
    & Deletes all absent units $\units\setminus S$ \\
  $\rhopad$
    & Replaces absent units with length-matched neutral placeholders \\
  $\pi$
    & Feasible permutation whose prefixes lie in $\feas$ and respect
      $\hier$-contiguity \\
  $\Pi(\units)$
    & Set of all permutations of $\units$ \\
  $\Pi_{\mathrm{feas}}(\dep,\hier)$
    & Set of feasible permutations under $\dep$ and $\hier$ \\
  $\mu$
    & Declared distribution over feasible permutations \\
  $\prefixset{i}{\pi}$
    & Units preceding $i$ in permutation $\pi$ \\

  \addlinespace
  \multicolumn{2}{@{}l}{\emph{Tasks, value, and valuation}}\\
  $\tasks,\ t$
    & Target task distribution and a task $t\sim\tasks$ \\
  $T_A,T_B$
    & Disjoint attribution and held-out confirmation panels, respectively \\
  $M_A,M_B$
    & Panel sizes, $M_A=|T_A|$ and $M_B=|T_B|$ \\
  $\agentop$
    & Frozen target agent \\
  $\score$
    & Programmatic benchmark-native task verifier \\
  $v_\rho(S,t)$
    & Per-task score:
      $\score\!\left(\agentop(\rho(S),t)\right)$ \\
  $\val_\rho(S)$
    & Coalition value under $\rho$, obtained by averaging
      $v_\rho(S,t)$ over the attribution panel $T_A$ \\
  $\val_{\rhodel}(\varnothing)$
    & Bare-agent value with no skill \\
  $\phi_{i,\rho}(\mu)$
    & Expected marginal contribution of execution unit $i$ under
      renderer $\rho$ and feasible-order distribution $\mu$ \\
  $\theta_m$
    & Trigger value:
      $\val_{\rhodel}(\{m\})-\val_{\rhodel}(\varnothing)$ \\
  $\Lambda$
    & Content lift:
      $\val_{\rhodel}(\units)-\val_{\rhodel}(\{m\})$ \\

  \addlinespace
  \multicolumn{2}{@{}l}{\emph{Budgeted estimator}}\\
  $K$
    & Number of sampled feasible permutations \\
  $W_k,\ b$
    & Stratified task window for permutation $k$ and its size,
      $W_k\subseteq T_A$ and $|W_k|=b$ \\
  $\bar v_\rho(S,W_k)$
    & Mean score of coalition $S$ on window $W_k$ under $\rho$ \\
  $\hat\phi_{i,\rho}$
    & Budgeted estimate of $\phi_{i,\rho}(\mu)$ \\
  $\hat\theta_m$
    & Estimated trigger value \\
  $\tau$
    & Truncation tolerance relative to the full-skill anchor \\
  $\gamma$
    & Average fraction of intermediate prefixes evaluated before truncation \\
  $R,\ R_{\mathrm{full}}$
    & Rollout cost of the budgeted and full estimators, respectively \\
  \bottomrule
  \end{tabular}
  \end{table}


\section{Compiler Specification}
  \label{app:compilation}

  The compiler $\comp$ maps a skill artifact to
  $\game=(\units,\dep,\hier)$ through four deterministic passes:
  unitization, dependency extraction, graph normalization, and
  dependency--hierarchy reconciliation. Each output unit records its source
  span, and each dependency edge records its rule identifier, source and target
  units, and triggering evidence. The compiled structure can therefore be
  audited against the original artifact.

\begin{table}[tb!]
  \centering
  \caption{Dependency rules enabled in the compiler.}
  \label{tab:rules}
  \small
  \begin{tabular}{@{}llp{0.44\textwidth}p{0.25\textwidth}@{}}
  \toprule
  Rule & Name & Condition & Emitted edge and evidence \\
  \midrule
  R1 & \textsc{link} &
  Unit $u$ contains a Markdown link resolving to unit or section $x$. &
  $u\to\operatorname{primary}(x)$; link and resolved target. \\

  R2 & \textsc{path} &
  Unit $u$ contains a path naming resource unit $r$. &
  $u\to r$; matched path. \\

  R3 & \textsc{heading-ref}$^\dagger$ &
  Unit $u$ contains a verbatim heading mention of at least two words and eight
  characters. &
  $u\to\operatorname{primary}(x)$; matched heading. \\

  R4 & \textsc{def-use} &
  A supported code block in $u$ uses a symbol defined in unit $v$. &
  $u\to v$; symbol and definition/use locations. \\

  R5 & \textsc{example-ref} &
  A workflow unit $u$ explicitly invokes an example resource $r$ exposed by
  the skill manifest. &
  $u\to r$; matched example identifier and resource path. \\

  R6 & \textsc{trigger} &
  $i$ is an execution unit. &
  $i\to m$; structural convention. \\

  R7 & \textsc{table-cont} &
  A unit $u$ contains headerless table rows continuing a table introduced by
  unit $v$. &
  $u\to v$; row and header locations. \\

  R8 & \textsc{list-cont}$^\dagger$ &
  An unordered top-level item $u$ immediately refines the nearest preceding
  ordered item $v$ in the same section. &
  $u\to v$; the two list markers. \\

  R9 & \textsc{lead-in}$^\dagger$ &
  Unit $\ell$ ends in a colon and immediately introduces an uninterrupted run of
  units $u_1,\ldots,u_k$. &
  $u_j\to\ell$ for each $j$; lead-in and run spans. \\
  \bottomrule
  \end{tabular}

  \raggedright\footnotesize
  $^\dagger$ Deterministic surface heuristic. We ablate R3, R8, and R9 separately
  because their evidence is weaker than explicit links, paths, and code-level
  def--use relations. R5 is enabled but emits no edge in the reported experiments
  because none of the evaluated artifacts exposes a separate example resource.
  \end{table}

  \subsection{Unitization and Scaffolding}

  The frontmatter is compiled as the trigger unit $m$. In the Markdown body, each
  list marker at indentation zero opens a body unit. Its indented bullets, code
  blocks, tables, and continuation lines remain attached until the next
  indentation-zero item, heading, protected-region boundary, or end of file;
  blank lines alone do not terminate a unit. Each auxiliary file exposed by the
  skill manifest is compiled as a resource unit.

  A region marked by the producing optimizer as atomically rewritten is compiled
  as one protected unit. Any top-level content not covered by the preceding rules
  is also compiled as one protected unit and flagged in the compiler report.
  This fallback ensures that every non-scaffold source span belongs to exactly
  one unit.

  Headings and horizontal separators are scaffold rather than players. Under
  $\rhodel$, a scaffold span is rendered exactly when at least one unit in its
  subtree survives. Under $\rhopad$, an otherwise absent scaffold span is
  replaced by a length-matched neutral span. Thus scaffold rendering is a
  deterministic function of the coalition rather than an additional player.

  \subsection{Dependency Extraction}

  For every emitted edge $u\to v$, unit $u$ requires unit $v$. Table~\ref{tab:rules}
  lists the rules enabled in our experiments.

  \paragraph{Worked example.}
  Fig.~\ref{fig:skillgraph} illustrates one compiler trace. The instruction
  unit $n_4$ uses \texttt{OUTPUT\_PATH}, whose definition occurs in the resource
  unit $s$ corresponding to \texttt{template.py}. Rule R4 therefore emits the
  edge $n_4\to s$. Its audit record stores the rule identifier, symbol name, and
  definition and use locations. The remaining edges are recorded analogously
  under the rules in Table~\ref{tab:rules}.

 \subsection{Graph Normalization}

  Three conventions are applied after edge extraction. First, content-derived
  rules do not emit edges from $m$, preserving $m$ as the unique minimal trigger
  unit. Second, a reference to a section is resolved to its primary unit: the
  section's body unit when present, and otherwise its first descendant unit.
  This convention applies to R1 and R3.

  Third, every strongly connected component of $\dep$ is contracted into an
  atomic composite unit. Its member spans remain in their original document
  positions but are retained or removed jointly. The compiler records all
  members of each contraction because contraction reduces valuation resolution.
  The hierarchy is then recomputed before compatibility checking.

  \subsection{Dependency--Hierarchy Compatibility}
  \label{app:conflict}

  After cycle contraction, $\dep$ is a DAG, but it may still conflict with
  hierarchical contiguity. For each internal hierarchy node $h$, let its children
  form candidate blocks. Construct a quotient graph $Q_h$ over these blocks:
  if a unit under child $c$ depends on a unit under child $c'$, add the precedence
  arc $c'\to c$. Self-arcs are ignored. Assuming each child is internally
  feasible, its children can be ordered as contiguous blocks if and only if
  $Q_h$ is acyclic.

  The compiler checks hierarchy nodes bottom-up. At node $h$, it initially treats
  each child subtree as one local sampling block. Whenever $Q_h$ is cyclic, every
  child block on a deterministic witnessing cycle is expanded into its repaired
  child blocks, and the quotient graph is rebuilt. Expansion relaxes contiguity
  only at the conflicting boundary: active descendant blocks remain intact, and
  ancestor constraints are unchanged. The process produces a set of
  repaired sampling blocks $C_h^*$ whose quotient graph is acyclic. Every
  expansion and its witnessing cycle are recorded in the compiler report.

  This procedure terminates because each expansion replaces a non-leaf block by
  blocks strictly lower in the finite hierarchy. In the limiting case,
  $C_h^*$ consists of individual units; its quotient graph is then a subgraph of
  the acyclic dependency graph. The repaired hierarchy therefore admits at least
  one feasible order while retaining every contiguity constraint that was not
  implicated in a conflict.

\section{The Feasible-Order Sampler}
  \label{app:sampler}

Section~\ref{sec:game} defines unit value with respect to a distribution $\mu$
  over feasible insertion orders. This appendix specifies the sampler inducing
  the experimental $\mu$ and establishes its soundness, full support, and
  nonuniformity. Throughout this section, $\hier$ denotes the repaired hierarchy
  returned by the compatibility pass in Appendix~\ref{app:conflict}.

  \subsection{Feasible Orders}

  Write an order as $\pi=(\pi_1,\ldots,\pi_n)$ and let
  \[
  S_j^\pi=\{\pi_1,\ldots,\pi_j\}
  \]
  be its $j$th prefix. We say that $\pi$ \emph{respects $\hier$} if the units in
  every active hierarchy block appear consecutively. A feasible order must both
  respect $\hier$ and have only dependency-closed prefixes. Thus,
  \begin{equation}
  \label{eq:feasible-orders}
  \Pi_{\mathrm{feas}}(\dep,\hier)
  \coloneqq
  \left\{
  \pi\in\Pi(\units)\ \middle|\
  \substack{
  S_j^\pi\in\feas\quad (j=1,\ldots,n),\\
  \pi\text{ respects }\hier
  }
  \right\}.
  \end{equation}

  The two conditions have distinct roles. Prefix feasibility ensures that a unit
  never appears before a unit it requires, while hierarchy compatibility prevents
  units from different document blocks from being arbitrarily interleaved.
  Because every execution unit has the dependency $i\to m$, every feasible order
  begins with the trigger unit $m$, consistent with the trigger/execution
  decomposition in Sec.~\ref{sec:game}.

  \subsection{Recursive Block Sampler}

  At an internal hierarchy node $h$, let $C_h^*$ be the repaired sampling blocks
  returned by the compatibility pass. We construct a block graph $Q_h$ over
  $C_h^*$: if some unit below block $c$ depends on a unit below block $c'$, we add
  the precedence arc $c'\to c$. Thus every edge points from a prerequisite block
  to a block that requires it. The compatibility pass guarantees that every such
  block graph is acyclic.

  \begin{algorithm}[t]
  \caption{\textsc{SampleOrder}$(h)$}
  \label{alg:sampler}
  \begin{algorithmic}[1]
  \REQUIRE repaired hierarchy node $h$, dependency DAG $\dep$
  \ENSURE an order of all units in $L(h)$
  \IF{$h$ is a unit}
    \RETURN $[h]$
  \ENDIF
  \STATE $R\leftarrow C_h^*$, \quad $\pi\leftarrow[\,]$
  \STATE build the block graph $Q_h$ over $C_h^*$
  \WHILE{$R\neq\varnothing$}
    \STATE $A\leftarrow
      \{c\in R:\nexists\,c'\in R\text{ with }c'\to c\text{ in }Q_h\}$
    \STATE draw $c$ uniformly from $A$
    \STATE $\pi\leftarrow\pi\mathbin\Vert\textsc{SampleOrder}(c)$
    \STATE $R\leftarrow R\setminus\{c\}$
  \ENDWHILE
  \RETURN $\pi$
  \end{algorithmic}
  \end{algorithm}

  The eligible set $A$ contains exactly the remaining sampling blocks whose
  prerequisite blocks have already been emitted. Since $Q_h$ is acyclic, $A$ is
  non-empty at every iteration. The order supplied to
  Algorithm~\ref{alg:estimator} is
  $\textsc{SampleOrder}(\operatorname{root}(\hier))$.

    \subsection{Sampler Correctness}

  \begin{proposition}[Validity and full support]
  \label{prop:sampler-correctness}
  Assume that the repaired block graph at every hierarchy node is acyclic after
  the compatibility pass. If $\mu$ is induced by Algorithm~\ref{alg:sampler}, then
  \[
  \operatorname{supp}(\mu)
  =
  \Pi_{\mathrm{feas}}(\dep,\hier).
  \]
  \end{proposition}

  \begin{proof}
  We first show that every sampled order is feasible. The recursive call emits
  all units under a selected sampling block before selecting another block at the
  same repaired boundary, so the resulting order respects $\hier$. For any
  dependency $u\to v$,
  consider the lowest hierarchy node placing $u$ and $v$ in different sampling
  blocks. Its block graph contains an arc from the block of $v$ to the block of
  $u$. The sampler therefore emits the block containing $v$ before the block
  containing $u$. Dependencies internal to a block are handled recursively.
  Hence every prerequisite precedes the unit requiring it, and every prefix lies
  in $\feas$.

  Conversely, let $\pi\in\Pi_{\mathrm{feas}}(\dep,\hier)$. Because $\pi$ respects
  $\hier$, it induces an ordering of the repaired sampling blocks at every
  internal node. Whenever the next block $c$ in this ordering is considered,
  $c$ cannot depend on an unplaced block $c'$; otherwise $\pi$ would place a
  dependent block before one of its prerequisites. Thus $c$ belongs to the
  sampler's eligible set.
  Every eligible block is selected with positive probability, so the finite
  sequence of choices producing $\pi$ also has positive probability.
  \end{proof}

  The first inclusion guarantees that Algorithm~\ref{alg:estimator} evaluates
  only valid partial skills. The reverse inclusion shows that the sampler does
  not structurally exclude any feasible order. It does not imply that feasible
  orders receive similar probabilities or will all be observed under a finite
  budget; the induced nonuniform distribution is characterized next.

\subsection{Why $\mu$ Is Not Uniform}

  Algorithm~\ref{alg:sampler} samples uniformly from the locally eligible blocks,
  but this does not imply a uniform distribution over complete feasible orders.
  Consider a flat hierarchy over $\{x,y,z\}$ with the dependency $y\to x$, so
  $x$ must precede $y$. The feasible orders are
  \[
  xyz,\qquad xzy,\qquad zxy.
  \]
  Initially, the eligible units are $\{x,z\}$. Choosing $z$ forces the order
  $zxy$, whereas choosing $x$ leaves two possible continuations. Therefore,
  \[
  \mu(zxy)=\frac12,
  \qquad
  \mu(xyz)=\mu(xzy)=\frac14,
  \]
  instead of $1/3$ for every order. A globally uniform sampler would choose $x$
  with probability $2/3$, proportional to its two feasible continuations, rather
  than uniformly from the current eligible set.

  Thus, the sampler induces a specific nonuniform distribution $\mu$. Changing
  the sampler would change the value in Eq.~\ref{eq:value}, so $\mu$ is part of
  the estimand rather than merely an implementation detail.

\section{Estimator Analysis}
  \label{app:estimator}
  Once the value definition is fixed, the estimator is the Monte Carlo average
  of marginal gains over feasible orders. The design question is therefore not
  how to construct a new estimand, but how to estimate the declared value under
  an expensive rollout oracle without corrupting each marginal comparison. We
  address this by binding each sampled order to one shared task window and
  evaluating every prefix of that order on the same tasks.

  \begin{table}[tb!]
  \centering
  \caption{Benchmark statistics and evaluation protocols. Split sizes follow the
  released SkillOpt manifests. ``Max turns'' denotes the per-episode interaction
  limit specified by each environment.}
  \label{tab:datasets}
  \small
  \begin{tabular}{llrrrrl}
  \toprule
  Benchmark & Task type & Train & Val & Test & Max turns & Reported metrics \\
  \midrule
  LiveMath & math theorem MCQ & 35 & 18 & 124 & 1 & accuracy (EM) \\
  OfficeQA & doc-grounded agentic QA & 50 & 24 & 172 & 24 & EM / F1 \\
  SpreadsheetBench & spreadsheet manipulation & 80 & 40 & 280 & 30
    & cell pass (soft/hard) \\
  ALFWorld & embodied household tasks & 39 & 18 & 134 & 50
    & task success \\
  \bottomrule
  \end{tabular}
  \end{table}
\begin{table}[t]
  \centering
  \caption{Disjoint task panels allocated by one seeded shuffle without
  replacement. Panel~A is used for attribution and ranking; panel~B is used only
  for held-out confirmation. The listed sizes satisfy
  $|T_A|+|T_B|\le|\mathrm{test}|$.}
  \label{tab:panels}
  \footnotesize
  \begin{tabular}{lrrr}
  \toprule
  Benchmark & Test & Panel A & Panel B \\
  \midrule
  LiveMath & 124 & 60 & 62 \\
  OfficeQA & 172 & 86 & 86 \\
  SpreadsheetBench & 280 & 80 & 140 \\
  ALFWorld & 134 & 60 & 64 \\
  \bottomrule
  \end{tabular}
  \end{table}

  \subsection{One Order Gives a Chain of Execution Marginals}

  Every feasible order begins with the trigger unit $m$. Write its execution-unit
  suffix as $(i_1,\ldots,i_{n-1})$ and define
  \[
    S_0=\{m\},\qquad
    S_j=\{m,i_1,\ldots,i_j\},\qquad
    S_{n-1}=\units.
  \]
  The marginal of $i_j$ on the full attribution panel is
  \[
    \val_\rho(S_j)-\val_\rho(S_{j-1}).
  \]
  Averaging this quantity over feasible orders gives
  $\phi_{i_j,\rho}(\mu)$. Thus one sampled order yields one marginal sample for
  every execution unit. The trigger is not included in this chain: its value is
  estimated separately as
  $\theta_m=\val_{\rhodel}(\{m\})-\val_{\rhodel}(\varnothing)$, avoiding the
  semantically ambiguous padded empty artifact.

  \subsection{Replacing the Attribution Panel by a Window}

  Evaluating every prefix on all $M_A$ tasks in $T_A$ is expensive. For each
  sampled order $\pi_k$, we independently draw a stratified task window
  $W_k\subseteq T_A$ of size $b\ll M_A$. The window design is independent of
  $\pi_k$ and gives every task equal marginal inclusion probability. We estimate
  \[
    \bar v_\rho(S,W_k)
    =\frac{1}{b}\sum_{t\in W_k}v_\rho(S,t).
  \]
  Consequently, for every fixed coalition $S$,
  \[
    \E_{W_k}\!\left[\bar v_\rho(S,W_k)\right]
    =\frac{1}{M_A}\sum_{t\in T_A}v_\rho(S,t)
    =\val_\rho(S).
  \]
  Applying this identity to both sides of each marginal gives
  \[
    \hat\phi_{i,\rho}
    =\frac{1}{K}\sum_{k=1}^{K}
    \left[
      \bar v_\rho(\prefixset{i}{\pi_k}\cup\{i\},W_k)
      -\bar v_\rho(\prefixset{i}{\pi_k},W_k)
    \right],
  \]
  and, when truncation is disabled,
  \[
    \E[\hat\phi_{i,\rho}]=\phi_{i,\rho}(\mu).
  \]
  The same windows give the trigger estimator
  \[
    \hat\theta_m
    =\frac{1}{K}\sum_{k=1}^{K}
    \left[
      \bar v_{\rhodel}(\{m\},W_k)
      -\bar v_{\rhodel}(\varnothing,W_k)
    \right],
  \]
  which is unbiased for the panel-level trigger value.

  \subsection{Why Each Chain Uses One Shared Window}

  The two coalitions in a marginal must be evaluated on the same tasks. Otherwise
  their difference mixes the unit effect with differences in task difficulty.
  With a shared window $W$,
  \[
    \bar v_\rho(S\cup\{i\},W)-\bar v_\rho(S,W)
    =\frac{1}{b}\sum_{t\in W}
    \left[v_\rho(S\cup\{i\},t)-v_\rho(S,t)\right].
  \]
  Easy and difficult tasks therefore cancel inside the paired difference. If
  $s_d^2$ denotes the finite-panel variance of the task-level marginal in
  brackets, uniform sampling without replacement gives
  \[
    \operatorname{Var}_{W}\!\left[
      \bar v_\rho(S\cup\{i\},W)-\bar v_\rho(S,W)
    \right]
    =\left(1-\frac{b}{M_A}\right)\frac{s_d^2}{b}.
  \]
  Independent windows would instead retain the two level-score variances and
  lose their positive task-difficulty covariance.

  Sharing $W_k$ across the entire chain also preserves exact telescoping:
  \[
    \sum_{j=1}^{n-1}
    \left[
      \bar v_\rho(S_j,W_k)-\bar v_\rho(S_{j-1},W_k)
    \right]
    =\bar v_\rho(\units,W_k)-\bar v_\rho(\{m\},W_k).
  \]
  For deletion, adding the separately measured trigger difference gives the
  full-skill lift over the bare agent on every untruncated chain.

  \subsection{Noise-Gated Truncation}

  Before scanning a chain, we evaluate the full-skill anchor
  $\bar v_\rho(\units,W_k)$. If a prefix $S_j$ satisfies
  \[
    \left|
      \bar v_\rho(S_j,W_k)-\bar v_\rho(\units,W_k)
    \right|\le\tau,
  \]
  we stop the chain and assign zero marginal contribution to its remaining
  execution units. The omitted window-level marginals telescope to
  \[
    \sum_{\ell>j}
    \left[
      \bar v_\rho(S_\ell,W_k)-\bar v_\rho(S_{\ell-1},W_k)
    \right]
    =\bar v_\rho(\units,W_k)-\bar v_\rho(S_j,W_k),
  \]
  whose absolute value is at most $\tau$. Thus $\tau$ bounds the aggregate
  suffix discarded by one operator on one sampled window.

  This is an aggregate, window-level statement. Individual omitted marginals can
  be larger and cancel within the suffix, and the bound does not directly apply
  to the full panel. We therefore report truncation as a controlled approximation
  and compare it with an untruncated control run rather than claiming a per-unit
  error bound.

\section{Experimental Details and Extended Analyses}
  \label{app:setting}
  This appendix documents the protocol underlying the main-text experiments and
  provides additional analyses of estimator stability, decision support, and
  hierarchical attribution. The main text asks whether \textsc{SkillSV} is
  faithful and actionable. The analyses below test whether those conclusions can
  instead be explained by task sampling, rollout budget, or a particular way of
  aggregating unit values.

  \subsection{Experimental Protocol}

  \noindent\textbf{Benchmarks and splits.}
  We evaluate four agentic benchmarks spanning different verification regimes and
  interaction horizons. \textbf{LiveMath}~\citep{he2026livemath} consists of
  mathematician-level multiple-choice problems answered in a single turn.
  \textbf{OfficeQA}~\citep{opsahlong2026officeqa} evaluates document-grounded
  enterprise question answering with up to $24$ tool turns.
  \textbf{SpreadsheetBench}~\citep{ma2024spreadsheetbench} requires spreadsheet
  manipulation over trajectories of up to $30$ turns.
  \textbf{ALFWorld}~\citep{shridhar2021alfworld} evaluates embodied household
  task completion over up to $50$ interaction steps. This range lets us test
  whether block valuation remains useful as evaluation moves from a single
  verified answer to a long agent trajectory. Table~\ref{tab:datasets} summarizes
  the released splits and evaluation protocols.

  Each benchmark exposes a programmatic verifier. The coalition utility is the
  primary scalar field fixed in that benchmark's run configuration; additional
  verifier outputs are reported as secondary metrics and are not combined
  post hoc.

  \noindent\textbf{Skill selection.}
  For each benchmark, we run Trace2Skill~\citep{trace2skill},
  TextGrad~\citep{yuksekgonul2024textgrad}, GEPA~\citep{agrawal2026gepa}, and
  SkillOpt~\citep{skillopt}. We select the skill with the highest
  \emph{validation} score and freeze the selected artifact before constructing
  the test panels. Consequently, neither attribution scores nor held-out pruning
  results participate in skill selection. \textsc{SkillSV} receives only the
  resulting skill directory and does not inspect its optimizer or optimization
  history.

  \noindent\textbf{Target agent and scoring.}
  Every coalition is executed by the same frozen GPT-5.5 model snapshot, with
  identical decoding, tool, and interaction settings. Each rollout runs in an
  isolated sandbox and is scored by the benchmark's programmatic verifier.
  Therefore, differences between two coalition scores arise only from the
  rendered skill and rollout stochasticity, rather than model updates or
  cross-rollout state. The exact endpoint identifier, decoding parameters, retry
  policy, and environment configuration are fixed in the released run manifest.

  \noindent\textbf{Disjoint screening and confirmation panels.}
  We apply one seeded shuffle (seed~$42$) to each test split and allocate panels
  $T_A$ and $T_B$ without replacement; thus the panels are disjoint by
  construction. Panel~A supplies the task windows $W_k$ used to estimate unit
  values and induce rankings. Panel~B is never used to select the skill, tune estimator
  hyperparameters, rank units, or propose edits. It is reserved for reported
  held-out anchors, pruning curves, refinement evaluation, and audits. Panel
  sizes are fixed before inspecting coalition scores. Panel~B is
  evaluated in full at every held-out checkpoint, making its size the main
  determinant of pruning-curve uncertainty. Table~\ref{tab:panels} reports the
  pre-specified panel sizes.

   \begin{figure}[tb!]
  \centering
  \includegraphics[width=\textwidth]{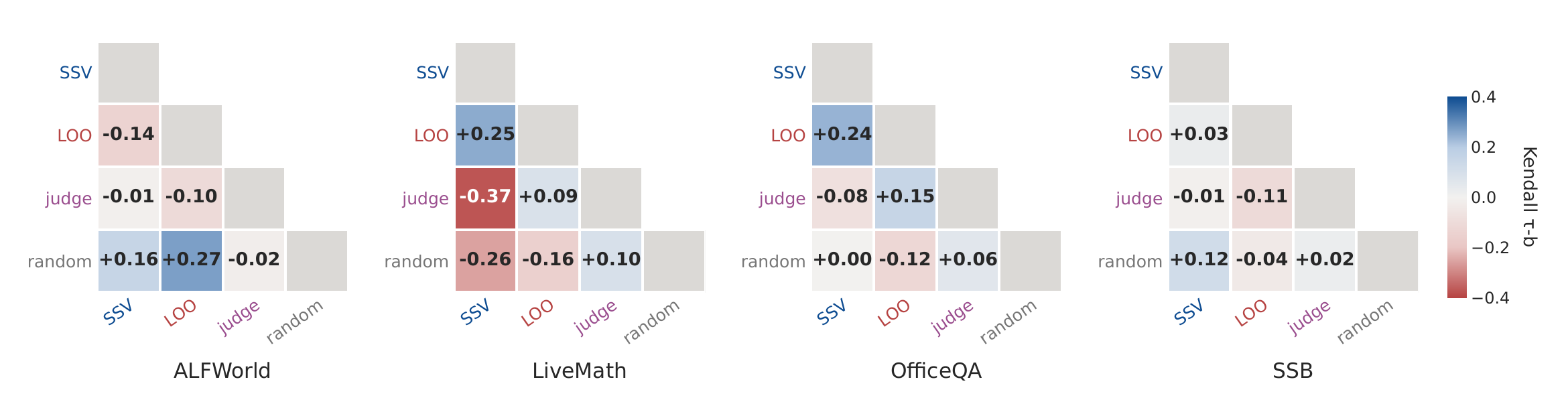}
  \caption{Pairwise agreement among attribution rankings. We report a tie-aware
  Kendall coefficient and assess significance against permutation null
  distributions that preserve each method's tie structure. Held-out pruning
  performance, rather than agreement itself, determines ranking utility.}
  \label{fig:rankagree}
  \end{figure}

  \noindent\textbf{\textsc{SkillSV} configuration.}
  We compile each selected skill at block granularity and draw $K=12$ feasible
  orders from the declared order distribution. Each order respects the dependency
  DAG and the repaired hierarchy. Independently for each order, we draw a
  stratified task window $W_k$ of $b=8$ tasks from $T_A$ and reuse it for every
  adjacent marginal on that prefix chain. The window draw is independent of the
  order and gives every task equal marginal inclusion probability.

  For a feasible coalition $S$, $\rhodel$ deletes units outside $S$, whereas
  $\rhopad$ replaces their spans with neutral placeholders matched to the target
  model's token length. Feasibility is enforced when coalitions and orders are
  constructed; the rendering operators do not independently expand a dependency
  closure. In contrast, the held-out pruning walks in the main text remove the
  reverse-dependency closure of a selected unit whenever this is necessary to
  keep the remaining coalition feasible.

  We use noise-gated truncation with $\tau=0.05$. A suffix is skipped only after
  the current prefix score is within $\tau$ of the full-skill anchor on the same
  task window. For OfficeQA, where both operators are evaluated, the suffix is
  truncated only when the gate holds for both operators. The tolerance bounds the
  aggregate omitted suffix on that sampled window; it does not bound every
  individual omitted marginal.

  We evaluate both $\rhodel$ and $\rhopad$ on OfficeQA, our two-operator audit,
  and use $\rhodel$ on the other benchmarks. Thus all four benchmarks support the
  net-value and pruning results, while conclusions about content value and
  context-occupancy cost are scoped to OfficeQA.

  \noindent\textbf{Uncertainty estimation.}
  Unit-value uncertainty is computed by resampling complete
  $(\pi_k,W_k)$ chains, preserving the dependence among all marginals obtained
  from the same order and task window. Held-out anchors and pruning curves instead
  use paired task-level bootstrap samples from panel~B. For comparisons between
  two pruning methods, the same bootstrap task sample is used to recompute both
  curves and their AUC difference. We use $1000$ bootstrap replicates throughout.

\begin{figure}[tb!]
  \centering
  \includegraphics[width=0.6\textwidth]{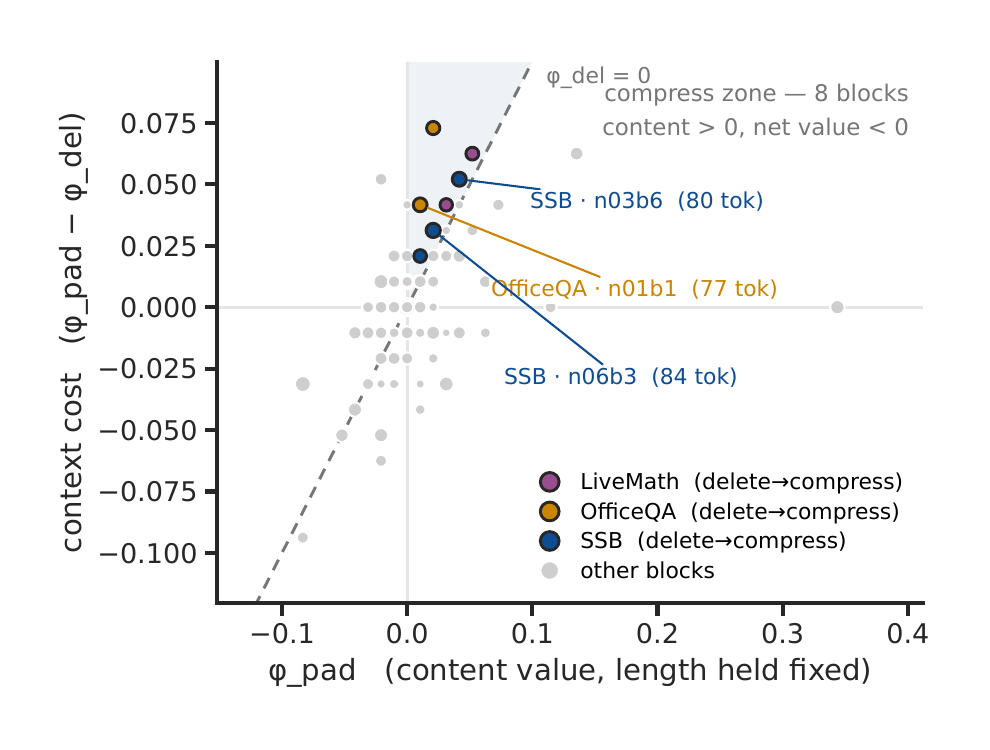}
  \caption{OfficeQA block-level decision map. The horizontal axis is content
  value $\phi_{i,\rhopad}$ and the vertical axis is context-occupancy cost
  $\phi_{i,\rhopad}-\phi_{i,\rhodel}$. The decomposition distinguishes content
  that should be retained from useful but overly costly content that should be
  compressed.}
  \label{fig:quadrant}
  \end{figure}

  \subsection{Estimator Calibration and Stability}

  \noindent\textbf{Selecting the task-window size.}
  Figure~\ref{fig:varb} evaluates the variance--budget trade-off underlying the
  choice $b=8$. For each candidate window size, we compare chain-coupled
  evaluation, in which both sides of every marginal use the same tasks, with an
  unpaired alternative using independently sampled tasks. Pairing removes the
  shared task-difficulty component and leaves only variation in task-level
  marginal effects. At the smallest windows, pairing reduces the measured
  variance to roughly one sixth of the unpaired alternative. At $b=M_A$, the two
  procedures coincide, providing an implementation consistency check. We choose
  $b=8$ as the smallest window that covers all required task strata while
  providing a stable truncation comparison.

  \begin{figure}[t]
  \centering
  \includegraphics[width=0.6\textwidth]{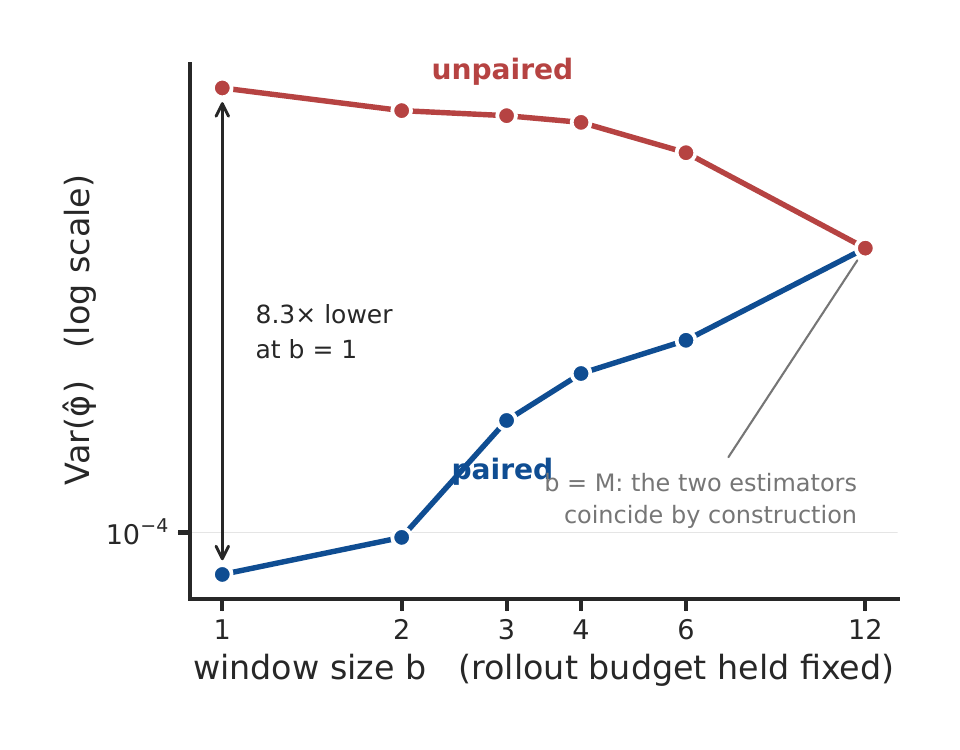}
  \caption{Variance--budget trade-off as task-window size $b$ varies.
  Chain-coupled evaluation uses the same task window on both sides of each
  marginal; the unpaired alternative samples them independently. The marked
  setting $b=8$ is used in all main experiments.}
  \label{fig:varb}
  \end{figure}

  \noindent\textbf{Calibrating truncation.}
  Window size and truncation tolerance cannot be selected independently: smaller
  windows make the saturation comparison noisier, whereas a larger tolerance can
  discard meaningful suffix effects. We calibrate $(b,\tau)$ on planted-value
  games and compare the selected $\tau=0.05$ configuration with an untruncated
  control run. In accordance with Appendix~\ref{app:estimator}, this audit
  evaluates aggregate and per-unit discrepancies separately; only the aggregate
  omitted suffix has a deterministic window-level bound.

  \subsection{Additional Attribution Diagnostics}

  \noindent\textbf{From values to keep--compress--delete decisions.}
  The pruning experiments in the main text use deletion values to test whether
  low-valued content can be removed. The two-operator OfficeQA audit provides a
  more refined decision view. Figure~\ref{fig:quadrant} places each block according
  to its content value $\phi_{i,\rhopad}$ and context-occupancy cost
  $\phi_{i,\rhopad}-\phi_{i,\rhodel}$. Positive-content blocks with small
  occupancy cost are natural retention targets, whereas positive-content blocks
  with substantial occupancy cost are candidates for compression rather than
  deletion. A block whose confidence interval includes zero on a decision axis,
  or whose magnitude lies below the estimator's reporting resolution, is treated
  as unresolved rather than assigned a categorical action. In this audit the
  occupancy effects are measurable but substantially narrower than the deletion
  values, so the decomposition is used as a targeted diagnostic rather than as a
  claim that context cost dominates these skills.

  \noindent\textbf{Relationship to baseline rankings.}
  Figure~\ref{fig:rankagree} compares the rankings induced by \textsc{SkillSV},
  Closure-LOO, the LLM judge, and random ordering, retaining ties rather than
  breaking them arbitrarily. We compare each observed coefficient with a
  permutation null that preserves the two rankings' tie structures and apply a
  Holm correction across benchmark--method comparisons. No method pair exceeds
  this corrected chance baseline. Low agreement does not by itself establish
  that one ranking is better; that evidence comes from the value-closure and
  held-out pruning experiments in the main text. Instead, this analysis shows
  that \textsc{SkillSV} is not merely a rescaled version of either learned or
  leave-one-out baselines.

  \noindent\textbf{Hierarchical localization of value.}
  Unit values can be aggregated through the hierarchy produced by the compiler.
  For an internal node, we sum the values of all descendant execution units;
  the node is not separately re-estimated. Figure~\ref{fig:valuetree} applies this
  aggregation to LiveMath. It complements the main-text concentration analysis by
  showing where the concentrated value resides in the document structure, rather
  than only how much value is captured by the top-ranked blocks.

  \begin{figure}[tb!]
  \centering
  \includegraphics[width=0.6\textwidth]{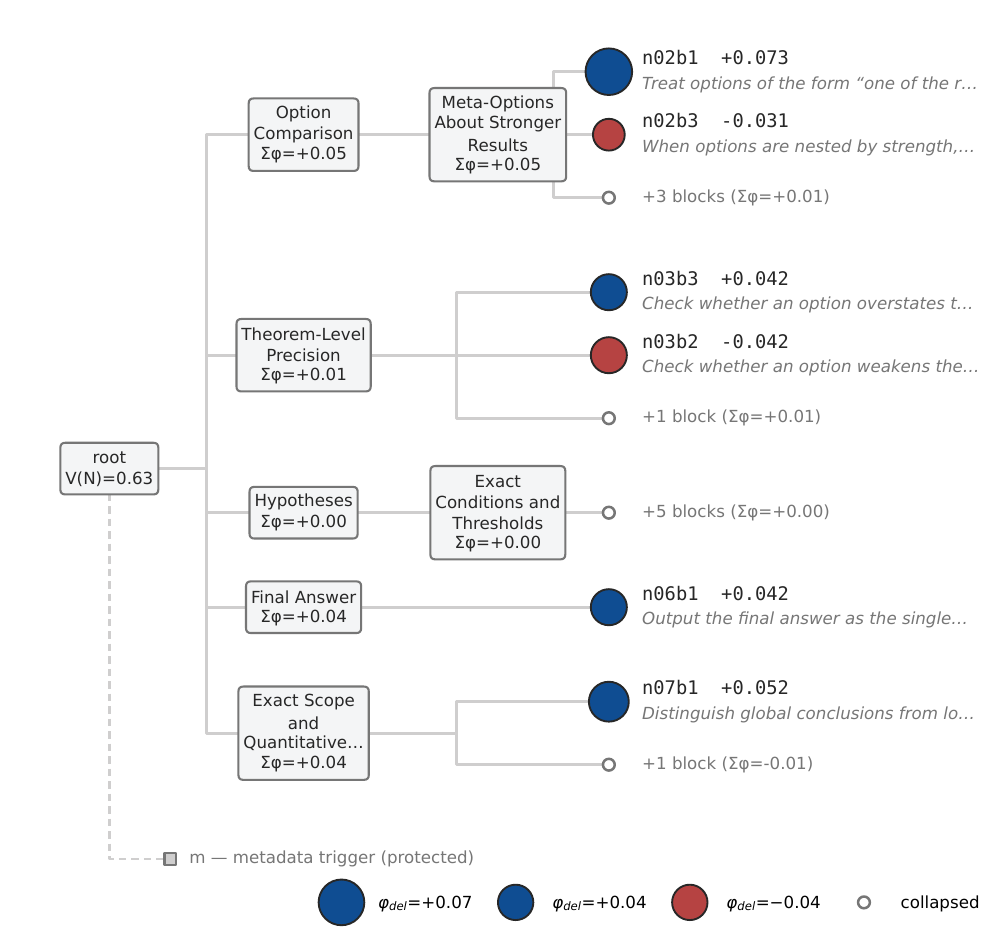}
  \caption{Hierarchical value tree for the LiveMath skill. Leaf nodes are
  compiled valuation units, and each internal-node value is the sum of its
  descendant unit values. The visualization localizes the value concentration
  reported in the main text to specific sections of the skill.}
  \label{fig:valuetree}
  \end{figure}

\end{document}